\documentclass[ijds,nonblindrev]{informs-ijds}

\OneAndAHalfSpacedXI
\usepackage{natbib}
 \bibpunct[, ]{(}{)}{,}{a}{}{,}%
 \def\bibfont{\small}%
\usepackage{graphicx,psfrag,epsf}
\usepackage[mathscr]{euscript}
\usepackage{enumerate}
\usepackage{url} % not crucial - just used below for the URL 
\usepackage{amsmath, amssymb}
\newtheorem{prop}{Proposition}
\usepackage{algorithm}
\usepackage[noend]{algpseudocode}
\usepackage{comment}
\usepackage{chngcntr}
\usepackage{apptools}
\AtAppendix{\counterwithin{prop}{section}}

\TheoremsNumberedThrough     % Preferred (Theorem 1, Lemma 1, Theorem 2)
\ECRepeatTheorems

\EquationsNumberedThrough    % Default: (1), (2), ...
\MANUSCRIPTNO{IJDS-****-****.**}

\begin{document}
%%%%%%%%%%%%%%%%

%\setcounter{page}{1}
%\DOI{10.1287/xxxx.0000.0000}%

\RUNAUTHOR{Wang and Paynabar}

\RUNTITLE{Maximum Covariance Unfolding Regression for Point Cloud Modeling}

% \TITLE{Bundling Information Goods of Decreasing Value}
\TITLE{Maximum Covariance Unfolding Regression: A Novel Covariate-based Manifold Learning Approach for Point Cloud Data}

\ARTICLEAUTHORS{%
\AUTHOR{Qian Wang}
\AFF{School of Industrial and Systems Engineering, Georgia Institute of Technology, Atlanta, GA 30332,  \EMAIL{qwang435@gatech.edu}} %, \URL{}}
\AUTHOR{Kamran Paynabar}
\AFF{School of Industrial and Systems Engineering, Georgia Institute of Technology, Atlanta, GA 30332, \EMAIL{kamran.paynabar@isye.gatech.edu}}
} % end of the block

\ABSTRACT{%
Point cloud data are widely used in manufacturing applications for process inspection, modeling, monitoring and optimization. The state-of-art tensor regression techniques have effectively been used for analysis of structured point cloud data, where the measurements on a uniform grid can be formed into a tensor. However, these techniques are not capable of handling unstructured point cloud data that are often in the form of manifolds. In this paper, we propose a nonlinear dimension reduction approach named Maximum Covariance Unfolding Regression that is able to learn the low-dimensional (LD) manifold of point clouds with the highest correlation with explanatory covariates. This LD manifold is then used for regression modeling and process optimization based on process variables. The performance of the proposed method is subsequently evaluated and compared with benchmark methods through simulations and a case study of steel bracket manufacturing.
}%

%\KEYWORDS{deterministic inventory theory; infinite linear programming duality;
%  existence of optimal policies; semi-Markov decision process; cyclic schedule}


\maketitle
%%%%%%%%%%%%%%%%%%%%%%%%%%%%%%%%%%%%%%%%%%%%%%%%%%%%%%%%%%%%%%%%%%%%%%

% Samples of sectioning (and labeling) in MNSC
% NOTE: (1) \section and \subsection do NOT end with a period
%       (2) \subsubsection and lower need end punctuation
%       (3) capitalization is as shown (title style).
%
%\section{Introduction.}\label{intro} %%1.
%\subsection{Duality and the Classical EOQ Problem.}\label{class-EOQ} %% 1.1.
%\subsection{Outline.}\label{outline1} %% 1.2.
%\subsubsection{Cyclic Schedules for the General Deterministic SMDP.}
%  \label{cyclic-schedules} %% 1.2.1
%\section{Problem Description.}\label{problemdescription} %% 2.

% Text of your paper here

\section{Introduction}

Point clouds, often referred to as high density data measured in a 3D coordinate space, contain rich spatial and geometric information about an object. Point cloud data play an increasingly important role in quality modeling and improvement of manufacturing processes where the shape of the manufactured product is geometrically complex \citep{inbook}. The dimensions and geometrical features of a part are often measured by laser scanners or touch-probe coordinate measuring machines (CMM). The resulting measurements, represented by a point cloud (see Figure \ref{fig:brktbend}), are consequently analyzed for quality inspection and monitoring, and process modeling and optimization. For example,  \cite{121791} and \cite{dryden} developed monitoring methods based on the distance of a point cloud to its nominal shape/geometry. For a detailed review of statistical process monitoring methods for point cloud data, see \cite{1641608}. Another body of literature focuses on process optimization and variation reduction for quality improvement by modeling the connection between process variables and point clouds. \cite{Colosimo2010} and \cite{colosimo2011analyzing} studied the impact of cutting speed and cutting depth in a turning process on the uniformity of manufactured cylinders using point clouds obtained from a CMM machine. \cite{doi:10.1080/00401706.2018.1529628} proposed a high-dimensional regression framework that models the relationship between the point-clouds and process variables, and used the estimated function to optimize a turning process. To build a regression model for point clouds, the general approach in the literature is to treat the response as a multivariate vector and regress it on the explanatory variables. A common two-step approach is comprised of reducing the dimensions using PCA first, and then, regressing the PC scores on the covariates. However, this approach fails to consider the relationship between response and covariates when finding a low-dimensional (LD) space. To take advantage of the covariates information, various methods including partial least square (PLS) \citep{helland1990partial}, sparse regression \citep{peng2010regularized}, function-on-scalar regression \citep{reiss2010fast}, and multiple tensor-on-tensor regression \citep{gahrooei2021multiple} have been used that can also deal with the high-dimensionality of point clouds.

The inherent assumption of the foregoing work is that the point cloud space is globally Euclidean. However, in many applications where the point clouds have complex shapes this assumption is not valid.  For example, the point cloud of a manufactured metal brackets shown in Figure \ref{fig:brktbend} \citep{dataset} has a complex geometry in which the distance between two points (e.g., one point on the curved area and the other on a flat arm) cannot globally be measured by the Euclidean distance. On the other hand, complex point clouds can often be represented by a manifold, a topological space that locally resembles Euclidean space near each point. That is, the distance of two near points can be measured by the Euclidean distance. This paper aims to propose a novel manifold-based learning and regression approach for modeling the variations of point clouds as a function of some scalars. The main analytical challenges in developing such a model include the non-Euclidean structure of point clouds, and their high-dimensionality that may lead to over-fitting. To address these challenges, we propose a manifold learning method that is capable of learning low-dimensional (LD) features containing important information about the relationship between the point clouds and explanatory variables, while preserving the properties of the original manifold.

Manifold Learning is a category of nonlinear dimension reduction techniques that aims to recover the low-dimensional ambient space of a data manifold using its local Euclidean structure. Many well-known manifold learning algorithms have been established in the literature. Examples include Isomap \citep{tenenbaum2000global}, Locally Linear Embedding \citep{roweis2000nonlinear}, Local Tangent Space Alignment \citep{zhang2004principal}, Laplacian Eigenmaps \citep{belkin2003laplacian}, Maximum Variance Unfolding \citep{666666}, and T-distributed Stochastic Neighbor Embedding \citep{van2008visualizing} that are used in different machine learning applications. For example, \cite{10.1145/1102351.1102455} and \cite{JMLR:v7:belkin06a} proposed a semi-supervised learning algorithm to learn from both labeled and unlabeled data using a novel data-dependent manifold regularization. \cite{5238610} developed an active learning framework on graphs based on Laplacian regularization. \cite{doi:10.1080/00401706.2015.1049751} applied manifold learning to discover the profile-to-profile variation pattern over high dimensional samples. \cite{doi:10.1080/00401706.2020.1772114} proposed to monitor the low-level spectrum of the Laplace–Beltrami (LB) operator on the point cloud manifold to control the shape of the manufactured point clouds and avoid the costly registration process used in other common approaches. \cite{doi:10.1080/00401706.2021.1883482} used dissimilarity-based manifold learning to obtain a manifold parameterization of variational patterns.

The problem setting in this paper, however, is different as we focus on modeling point cloud variations as a function of scalar covariates. One straightforward approach is to first use manifold learning to reduce the dimension of the response, followed by regressing the reduced dimensions on covariates. The drawback of this approach is that it fails to consider the relationship between covariates and the point cloud response during the dimension reduction process. To the best of our knowledge, this is the first work taking covariates information into consideration for learning LD manifolds. Inspired by the Maximum Variance Unfolding \citep{666666}, our proposed model unfolds the data manifold by maximizing the Frobenius norm of variance-covariance matrix between regressors (covariates) and the point cloud response. The method is thus coined as Maximum Covariance Unfolding (MCU). In this way, the local geometric information of the point cloud responses and corresponding regressors are considered in dimension reduction, leading to more effective learning and scalability.

%The geometry of brackets during the manufacturing process can be controlled by two process variables, namely bending position and bending force on one of its two arms.

%process modeling and optimization focuses on structured point cloud data, where coordinates are measured in some pre-defined grid which can be represented by a matrix or tensor $A=[a_(ijk)]$; where  each location (i,j,k) on the grid corresponds to a value, for example deviation from a nominal. The main assumption of 

%Differently, in this paper, we focus on a more general data type, named unstructured point cloud where such a pre-determined measurement grid does not exist. Unstructured point clouds can be widely seen in various manufacturing processes where manufactured parts have complex geometric shapes. 

 The remainder of the paper is organized as follows. Section 2 introduces the MCU as an optimization problem and elaborates the MCU regression approach. Section 3 validates the proposed methodology and compares its performance with benchmark methods by using two different simulated data sets. Section 4 includes a case study on process modeling and optimization in manufacturing of metal brackets. In section 5, the paper is concluded and future research directions are given.

\section{Maximum Covariance Unfolding Regression for Process Modeling and Optimization}

In this section, we elaborate the new manifold learning method, MCU, used for building a regression model for point clouds. The proposed MCU can achieve nonlinear dimensionality reduction on point clouds data while utilizing the covariates  information. 
Suppose there is a set of explanatory variables/covariates denoted by $x\in \mathbb{R}^P$ used to explain the variations of a point cloud response $y$. Also, assume each point cloud includes $n$ points denoted by $y \in \mathbb{R}^{n\times d}$, where each row indicates the coordinate of the point in $d$-dimensional space. After concatenating $y$ along the coordinates of the point cloud, it becomes a flattened vector $y \in \mathbb{R}^M$, where $M=nd$. Since $M$ is usually of high dimensions, dimension reduction is required before building the regression model. In this paper, we assume $y$ has nonlinear variational pattern over varying $x$, and lies on a low-dimensional manifold. Moreover, we assume that the data manifold is sufficiently smooth and densely sampled so that it has local Euclidean structure.
% It is important to note that in this paper, we regard the whole point cloud as a single data point. 
%that here we regard the whole point cloud as a single data point instead of a set of data points within the cloud. The structure of the manifold is usually highly non-linear and generally the geodesic distance of two points 
%along a manifold is not Euclidean. However, the space structure within a small neighborhood on the manifold is locally Euclidean. 

\subsection{MCU Manifold Learning and Regression}

Suppose a training sample of $N$ pairs of  $\{(x^{(i)},y^{(i)})|\ x^{(i)}\in \mathbb{R}^P, y^{(i)} \in \mathbb{R}^M, i=1,...,N\}$ is available, where $P$ is the number of process variables. Our goal is to reduce the HD manifold that includes the set of point clouds $\{y^{(i)}; i=1,...,N\}$ into some LD representation $\tilde{y} \in \mathbb{R}^{\tilde{M}}; \tilde{M}\ll M$, while taking the information of $x$ into consideration. Specifically, inspired by Maximum Variance Unfolding \citep{666666}, we aim to find $\tilde{y}^{(1)},...,\tilde{y}^{(N)}$ such that $||\tilde{y}^{(i)}-\tilde{y}^{(j)}||_2=||{y}^{(i)}-{y}^{(j)}||_2$ for all $(i,j) $ in the neighborhood graph $E$, and at the same time, the sum of squared covariance of all predictor-response pairs, i.e., $\sum_p\sum_m cov(x_p,\tilde{y}_m)^2$, is maximized. Here $x_p$ denotes the variable $p$ in $x$ and $\tilde{y}_m$ indicates the dimension $m$ of the LD space. Mathematically, this can be written as 

\begin{equation}
\label{eq:0}
\begin{aligned}
\max_{\tilde{y}^{(i)}}\quad& \sum_{p=1}^P\sum_{m=1}^{\tilde{M}} cov(x_p,\tilde{y}_m)^2 \\
\text{s.t.}\quad& ||\tilde{y}^{(i)}-\tilde{y}^{(j)}||_2^2=||{y}^{(i)}-{y}^{(j)}||_2^2\; \text{for}\  (i,j)\in E
\end{aligned}
\end{equation}
\\
The intuition behind this is to maximize the covariance information between explanatory variables $x_p$ and the transformed response variable $\tilde{y}_m$, while preserving the local Euclidean structure of the data manifold. The latter is achieved by ensuring that local distances in the original and transformed manifolds are the same.  We construct the neighborhood graph $E$ as the $k-$nearest neighborhood graph such that an edge $(i,j)$ belongs to $E$ if $i$ is among $j$'s k-nearest neighbors and vice versa. Although we use $k$-nearest neighborhood graph in this paper, the $\epsilon$-neighborhood can also be used for constructing the graph. The term $\emph{Maximum Covariance Unfolding}$ comes from the fact that the above optimization problem unfolds the data manifold to maximize the covariance metric, while preserving the local data connectivity.

 There is one potential obstacle in solving the above optimization problem, which also exists in the Maximum Variance Unfolding method. Insufficient sample size can violate the assumption that the manifold is densely sampled, leading to an unbounded objective function. Intuitively, the neighborhood graph of a sparsely sampled manifold can have more than one set of disconnected cliques. Thus, during the unfolding process, the disconnected cliques can be pushed away so far that it results in an unbounded objective function. In some applications, due to limited resources, the available sample size ($N$) is small or moderate, which results in sparsely sampled manifolds. %This poses a challenge to the development of a practical algorithm to be applied on the real dataset. 
 To deal with this issue, we pose an artificial upper bound on the sum of variances of each reduced response variable so that the disconnected nodes cannot be pushed away indefinitely. This can be translated to the following constraint.
 
 \begin{equation}
 \label{eq:1}
     \sum_{m=1}^{\tilde{M}} var(\tilde{y}_{m}) \leq C,
 \end{equation}
 
 %It is more ideal to put an upper bound separately for each reduced response entry i.e. $var(\tilde{y}_{m}) \leq C$ for every $m$. However, this makes the optimization problem harder to solve, so we didn't adopt this constraint. 
 where $C$ is a constant defined by the user. A heuristic approach for finding suitable $C$ is given in section \ref{sectionC}. In the following proposition, we prove that adding the constraint in (\ref{eq:1}) leads to an upper bound on the objective function.

 \begin{prop}
The objective function $\sum_{p=1}^P\sum_{m=1}^{\tilde{M}} cov(x_p,\tilde{y}_m)^2$ is bounded above if
$\sum_{m=1}^{\tilde{M}} var(\tilde{y}_{m}) \leq C$, where $C$ is a constant.
\end{prop}
 All proofs are given in the appendix. 
 
 Without loss of generality, we add an additional constraint on the LD embeddings $\tilde{y}$ that they are centered at 0, i.e., $\sum_{i=1}^N \tilde{y}^{(i)} = 0$. This will simplify solving the optimization problem. Therefore, the MCU problem formulation can be rewritten as 
 \begin{equation}
 \label{eq:2}
\begin{aligned}
\max_{\tilde{y}^{(i)}}\quad& \sum_{p=1}^P\sum_{m=1}^{\tilde{M}} cov(x_p,\tilde{y}_m)^2 \\
\text{s.t.}\quad& \sum_{i=1}^N \tilde{y}^{(i)} = 0\\
& \sum_{m=1}^{\tilde{M}} var(\tilde{y}_{m}) \leq C\\
& ||\tilde{y}^{(i)}-\tilde{y}^{(j)}||_2^2=||{y}^{(i)}-{y}^{(j)}||_2^2\; \text{for}\  (i,j)\in E
\end{aligned}
\end{equation}

In order to efficiently solve the above optimization problem, we will transform it to a semi-definite programming (SDP).

\subsubsection{Semi-definite Programming (SDP) Formulation}. 
We combine the data $(x^{(i)}, y^{(i)})$ across samples into the predictors matrix $X_{N\times P}$ and response matrix $Y_{N\times M}$, respectively. Similarly, we define $\tilde{Y}_{N\times \tilde{M}}$ as the matrix of LD embedded point clouds. Assuming that $x$ is centered around 0, we can reformulate problem (\ref{eq:2}) using the following proposition.

\begin{prop}
Assume $x$s are centered at 0, then Problem (\ref{eq:2}) is equivalent to 
\begin{equation}
\label{eq:3}
\begin{aligned}
\max_{\tilde{Y}}\quad& tr(XX^T\tilde{Y}\tilde{Y}^T) \\
\text{s.t.}\quad& tr(\mathbf{1}\mathbf{1}^T\tilde{Y}\tilde{Y}^T) = 0\\
& tr\left(\tilde{Y}\tilde{Y}^T\right) \leq (N-1)C\\
& (\tilde{Y}\tilde{Y}^T)_{ii}+(\tilde{Y}\tilde{Y}^T)_{jj} -2 (\tilde{Y}\tilde{Y}^T)_{ij}=({Y}{Y}^T)_{ii}+({Y}{Y}^T)_{jj} -2 ({Y}{Y}^T)_{ij}\; \ \text{for}\  (i,j)\in E
\end{aligned}
\end{equation}

where $\mathbf{1}$ is a vector of size $N\times 1$ consisting of all 1's, $A_{ij}$ denotes the $(i,j)^{th}$ entry of matrix $A$ and $tr(A)=\sum_{i}A_{ii}$ is the trace of matrix $A$.
\end{prop}
 %Note that the optimization problem (\ref{eq:3}) optimizes with respect to $\tilde{Y}$ actually through $\tilde{Y}\tilde{Y}^T$. Following the similar trick as in Maximum Variance Unfolding, 
 
 To transform the optimization problem in (\ref{eq:3}) to an SDP we define $Q\in \mathbb{R}^{N\times N} := \tilde{Y}\tilde{Y}^T$, which is a positive semi-definite matrix, i.e. $Q\succeq 0$, and  optimize (\ref{eq:3}) with respect to $Q$ instead of $\tilde{Y}$. To simplify the notation, we similarly define $P \in \mathbb{R}^{N\times N} := YY^T$. Hence, Problem (\ref{eq:3}) becomes
 \begin{equation}
\label{eq:4}
\begin{aligned}
\max_{Q\in \mathbb{R}^{N\times N}}\quad& tr(XX^TQ) \\
\text{s.t.}\quad& Q\succeq 0\\
& tr(\mathbf{1}\mathbf{1}^TQ) = 0\\
& tr(Q) \leq (N-1)C\\
& Q_{ii}+Q_{jj} -2 Q_{ij}=P_{ii}+P_{jj} -2 P_{ij}\; \ \text{for}\  (i,j)\in E.
\end{aligned}
\end{equation}

This is an SDP problem, which can be solved using interior point methods \citep{dikin1967iterative} by off-the-shelf optimization solvers such as MOSEK \citep{mosek}. 

\subsection{Dimension Reduction and Regression using MCU\label{sectionC}}

After solving the optimization problem in (\ref{eq:4}) and finding optimal $Q$, we need to recover $\tilde{Y}$ from $Q$. As $Q$ is positive semi-definite, using eigen decomposition, i.e., $Q = U\Sigma U^T$, we can recover $\tilde{Y}=U\Sigma^{\frac{1}{2}}$. To incorporate dimension reduction, while recovering $\tilde{Y}$, we can truncate matrices $U$ and $V$, such that only the $\tilde{M}$ largest eigenvalues and their corresponding eigenvectors are kept. In this case, $\tilde{Y}$ is obtained by $U_{1:\tilde{M}}\Sigma^{\frac{1}{2}}_{1:\tilde{M}}$ which is unique up to right multiplication by any orthogonal matrix. The eigenvalues in $\Sigma$ should be ordered from the largest on the upper left to the smallest on the lower right to make sure the most information is captured in $\tilde{Y}$.\\ 
% Since $Q = \tilde{Y}\tilde{Y}^T$, if we can find its eigen-decomposition $Q = U\Sigma U^T$, we can recover $\tilde{Y}$ as 

It should be noted that since $rank(Q)\leq min\{N,M\}$ and $\tilde{M}\leq rank(Q)$, it is true that $\tilde{M}\leq min\{N,M\}$. This means when number of observations $N<M$, the maximum dimensions that can be recovered from the unfolding process is limited by $N$. The remaining information is lost as a result of the fact that we only use the inner product information from the original data during the optimization. \\
 
After obtaining the LD manifold, $\tilde{Y}: N\times \tilde{M}$ by applying MCU, we can subsequently model it as a function of explanatory variables by regressing $\tilde{Y}$ on $X: N\times P$ using a coefficient matrix $B: P\times \tilde{M}$. i.e., $\tilde{Y} = X B + E$. This can be solved using the ordinary least square method $\hat{B} = (X^TX)^{-1}X^T\tilde{Y}$. A summary of the MCU manifold learning and regression algorithm is given in Algorithm \ref{alg:mcu}. Additionally, some practical considerations and details necessary to implement the algorithm are given below. 

\begin{algorithm}[H]
\caption{Maximum Covariance Unfolding Regression}\label{alg:mcu}
\begin{algorithmic}[1]
\Procedure{Maximum Covariance Unfolding Regression}{$X, Y, C, \tilde{M}, k, \lambda$}
    \State Standardize X column-wise with mean 0 and variance 1 
    \State Center Y with column mean 0 and scale Y by \(1/\frac{1}{M}\sum_{m=1}^M(\sqrt{var(y_m)})\)
    \State Solve the semi-definite programming in (\ref{eq:4}) and get optimal $Q$
    \State Get the eigen-decomposition $Q = U\Sigma U^T$ (eigenvalues ordered from large to small)
    \State Select the first $\tilde{M}$ columns of $U$ and $\Sigma$, and compute $\tilde{Y}=U_{1:\tilde{M}}\Sigma^{\frac{1}{2}}_{1:\tilde{M}}$
    \State Regress $\tilde{Y}$ on X without intercept and get regression coefficient $\hat{B}=(X^TX+\lambda I)^{-1}X^T\tilde{Y}$
\EndProcedure
\end{algorithmic}
\end{algorithm}

\begin{itemize}
    \item  \textbf{Standardization of $X$}. To satisfy the assumption used during the formulation that the sample mean of $X$ is zero, standardization of $X$ is needed so that each control variable $x_p$ has sample mean 0 and sample variance as. Scaling the sample variance is to ensure that the unfolding scales are the same in each direction of process variables and are not influenced by their magnitudes. As a result of centering $X$ and $\tilde{Y}$, the regression model has a 0 intercept.

\item \textbf{Choice of reduced dimensions for $\tilde{M}$}. $\tilde{M}$ could be chosen by the number of predictors or applying Otsu's method \citep{4310076} to the eigenvalues of $Q$ in the log scale. During the construction of $\tilde{Y}$ only eigenvalues above the Otsu threshold should be used. Log scale is preferred over the original scale because compared to the largest eigenvalues others may be too small, resulting in selecting only a few largest eigenvalues. In our numerical study, the reduced dimension is selected as the number of predictors.

\item \textbf{Choice of the neighbor number $k$}.
One can use cross-validation for choosing the neighbor number. Based on our numerical analysis and by evaluating the reconstruction and optimization errors, we find $k=4,5$ are the appropriate choices. 

\item \textbf{Find suitable upper bound C}. A reasonably large upper bound $C$ needs to be found to allow enough freedom for $Y$ to unfold. One heuristic approach is to first center $Y$ and then scale the whole matrix $Y$ by $1/\frac{1}{M}\sum_{m=1}^M(\sqrt{var(y_m)})$. In this way the average variance of response variables $var({y}_{m})$ becomes approximately $1$. It is reasonable to assume that each $var(\tilde{y}_{m})$ after unfolding becomes no smaller than $1$ as unfolding only increases the variance, and we can expect $tr(var(\tilde{Y}))=\sum_{m=1}^{\tilde{M}} var(\tilde{y}_{m}) \geq \tilde{M}$. We can then set $C = \alpha \tilde{M}$ where $\alpha \geq 1$. In this way, we transform the search of $C$ into the search of $\alpha$, independent of $Y$. This can facilitate the search for a suitable upper bound. In our numerical study we used $\alpha=100,000$ to allow more flexibility for unfolding.

\end{itemize}

\subsection{Predictive Model Optimization\label{sectionopt}}
Predictive optimization is a procedure to find the optimal values of the explanatory/control variables such that the predicted response is the closest to a nominal value/shape. Process optimization cannot be carried out in a direct way using MCU because the nonlinear relationship between $Y$ and $\tilde{Y}$ is lost during the optimization process and there is no explicit formula for modeling their connection. However, we can infer the LD nominal embedding ${\tilde{y}}_{nom}$ from the nominal $y_{nom}$ using an indirect optimization method. As the distances between neighboring points are preserved during the MCU, we can learn ${\tilde{y}}_{nom}$ from $y_{nom}$ using the training data $Y$ and $\tilde{Y}$ by solving the following optimization problem. 
\begin{equation}
{\tilde{y}}_{nom} = \underset{v\in \mathbb{R}^{M'}}{\mathrm{argmin}} \sum_{i\in E_{nom}}(||v - \tilde{y}^{(i)}||_2-||{y}_{nom}-{y}^{(i)}||_2)^2,\label{eq:po}
\end{equation}

where $E_{nom}$ is the index set of $k$-nearest neighbors of ${y}_{nom}$ among $y^{(i)}$ from ${Y}$.\\

As the optimal values of the explanatory variables, denoted by $x^*$, should result in a response close to its nominal value, we can replace ${\tilde{y}}_{nom}$ with $x\hat{B}$ in the optimization model in (\ref{eq:po}), where $\hat{B}$ is the estimated coefficients from MCU regression. Therefore, $x^*$ can be obtained by %One challenge is that optimal control values in practice are unknown and we cannot run MCU with missing covariates information to obtain $\tilde{y}^*$. \\
%The key observation is . Enjoying this property 
%Knowing how to transit from $y_{nom}$ to $\tilde{y}_{nom}$, we can easily do process optimization as follows:
$$
x^*  = \underset{x\in \mathbb{R}^P}{\mathrm{argmin}} \sum_{i\in E_{nom}}(||x\hat{B} - \tilde{y}^{(i)}||_2-||{y}_{nom}-{y}^{(i)}||_2)^2.
$$

This final optimization problem can be solved using stochastic global optimization algorithms, like Dual Annealing \citep{XIANG1997216}.

\section{Methodology Validation Using Simulations}
In this section, we use simulated data to validate the MCU manifold learning and regression methodology and compare its performance with two benchmarks, namely, principal component analysis (PCA) and minimum variance unfolding (MVU). In this study, we use two metrics to evaluate the performance of the algorithms. First, we adopt embedding reconstruction error from regression. Using the regression coefficient $\hat{B}$, we can reconstruct the LD embedding $\hat{\tilde{y}}$ from $x$ to $\hat{\tilde{y}}=x\hat{B}$. The relative reconstruction error (RRE) between ${\tilde{y}}$ and $\hat{\tilde{y}}$, defined by = $\|\tilde{y}-\hat{\tilde{y}}\|_2/\|\tilde{y}\|_2$, can be used for evaluation of the algorithm to see how much covariance between $x$ and $\tilde{y}$ can be explained by a linear regression model. The second evaluation metric is the Euclidean deviation $\|x_{nom}-\hat{x}_{nom}\|_2$ between the true optimal $x_{nom}$ and the estimated optimal covariates $\hat{x}_{nom}$ obtained from process optimization. We evaluate the performance of the foregoing methods using two commonly used simulated data sets; namely, the Swiss roll point cloud and penny image manifold.

\subsection{Swiss Roll}
To generate the Swiss roll dataset, we begin with a 2D base plane generated by \cite{preswissroll}, and then fold it into a Swiss roll as shown in Figures \ref{fig:swiss1} and \ref{fig:swiss2}. We define a function $f: \mathbf{R}^2\rightarrow \mathbf{R}^3$ to map each point $(x,y)$ in the base plane to one point in the 3D Swiss roll via two control (explanatory) variables $c_1$ and $c_2$. They are used to control the shape of the generated Swiss roll samples. Specifically, $c_1$ is used to control the size of the Swiss roll and $c_2$ is to control the period of spiral (i.e., the degree of folding) via the foldrate $r:= 1+c_2/10$. The mapping function used to generate the 3D Swiss rolls is given below.
$$
f(x,y) = \left((4/9*c_1+50/9)x\text{cos}(2\pi r(x-4)/12),\ y,\ (4/9*c_1+50/9)x\text{sin}(2\pi r(x-4)/12)\right)
$$

\begin{figure}[htbp]
\begin{minipage}[t]{0.5\linewidth}
\centering
\includegraphics[scale=0.2]{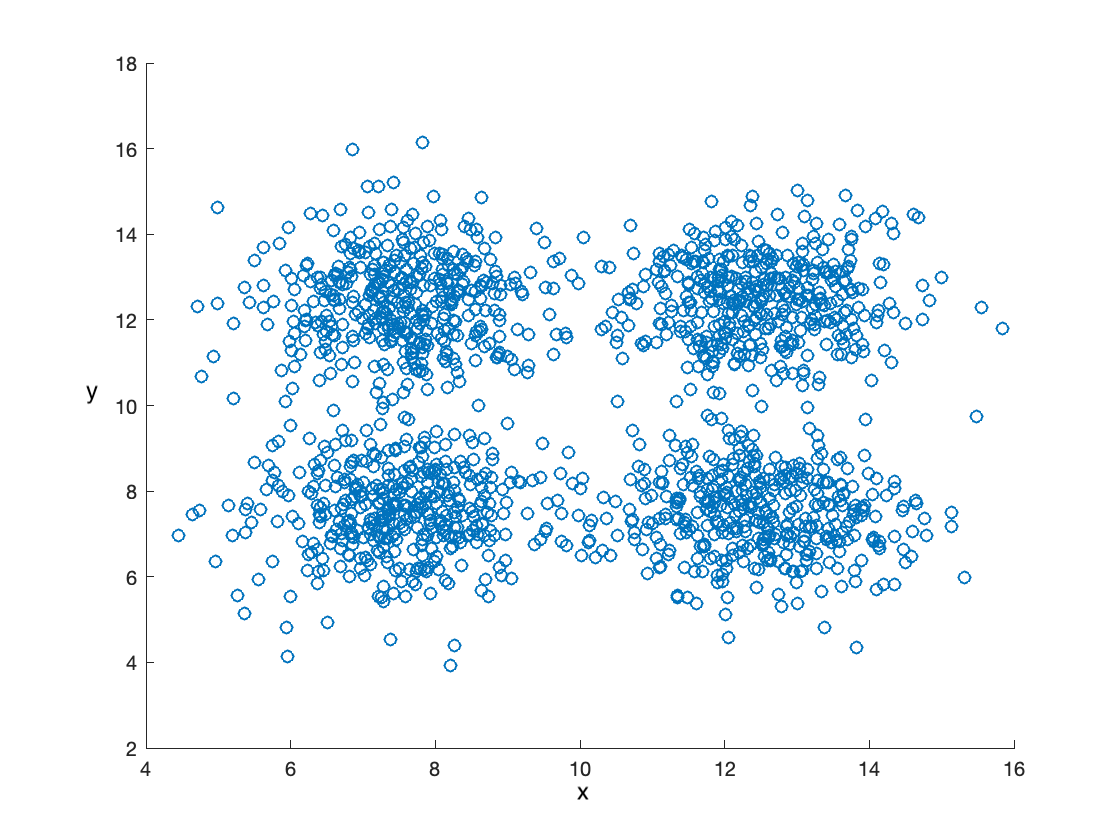}
    \caption{\label{fig:swiss1}Pre Swiss roll data}
\end{minipage}%
\begin{minipage}[t]{0.5\linewidth}
\centering
\includegraphics[scale=0.2]{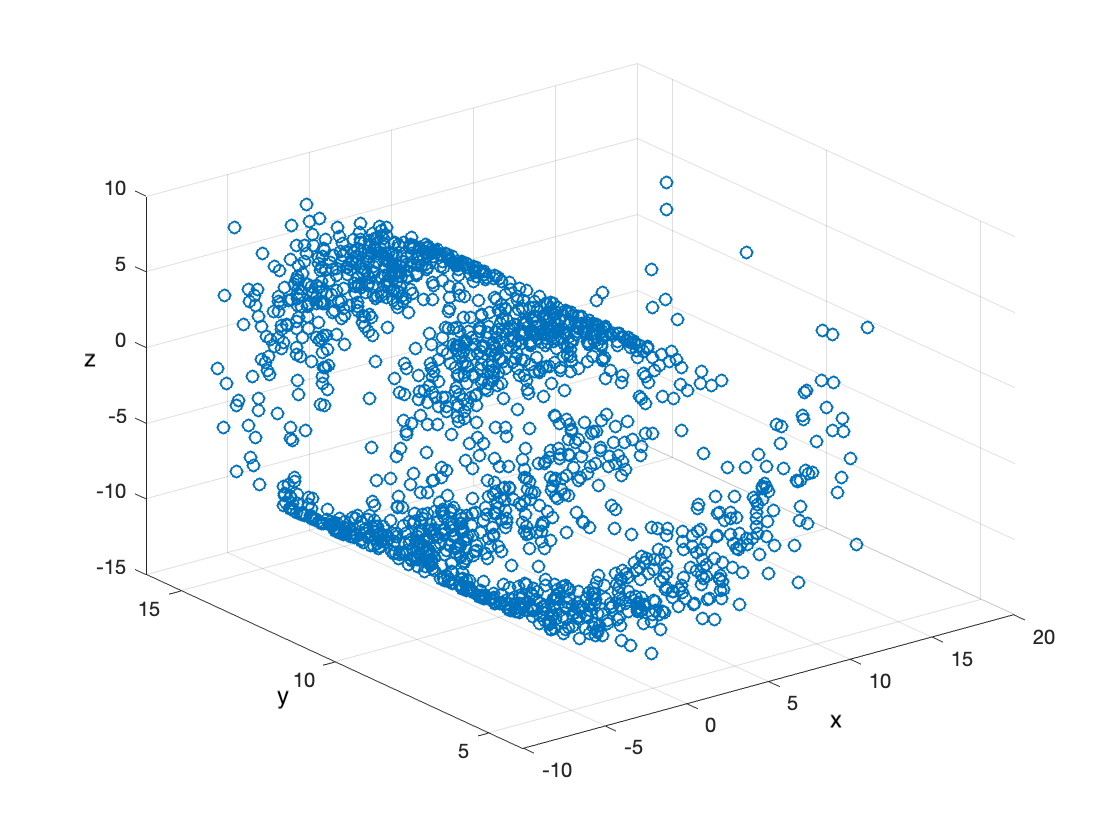}
    \caption{\label{fig:swiss2}Sample Swiss roll data}
\end{minipage}%
\end{figure}

We, then, generate 200 Swiss rolls by varying $c_1$ and $c_2$ randomly selected from $[1, 10]$. Each sample contains 1600 datapoints, so after concatenating the coordinates of these $1600$ datapoints in a row, the response matrix $Y$ is of size $N\times 4800$, where $N$ is the sample size. The explanatory variable matrix $X$ is of size $N\times 2$.  

To study the shape of the generated manifold data, we first apply a PCA to the set of generated Swiss roll samples and visualize the first 3 PC scores in Figure \ref{fig:swissmani}. The arrows on the plot show how the control variables $(c_1,c_2)$ change the coordinates of the samples on the generated manifold. As can be seen from the plot, the manifold has the shape of a half hollow cylinder, where $c_1$ controls the height of the cylinder and $c_2$ controls the circumference. From Figure \ref{fig:swissmani}, we can also see that the manifold has a nonlinear structure.

\begin{figure}[htbp]
\centering
\includegraphics[scale=0.6]{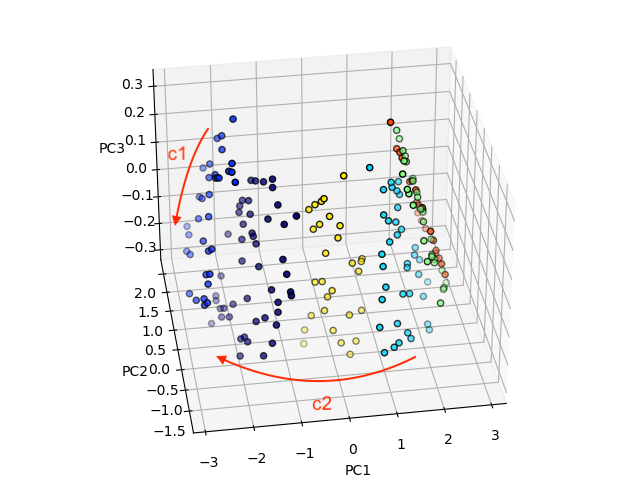}
    \caption{\label{fig:swissmani}A colormap(based on clustering) of manifold shape of Swiss roll data}
\end{figure}

% commenting out the following block
%\begin{comment}
%We run our algorithm using the hyperparameter setting as follows. 
%$$
% k = 4, \lambda = 10^{-10}, d = 2
%$$

%The degrees of freedom of the underlying data manifold should be 2 as a result of the number of control variables being 2, thus we choose $d$ as 2 in the hyperparameter selection.\\
%\end{comment}

Next, we apply the proposed MCU along with the benchmarks, i.e., PCA and MVU, and compare their performance. For the benchmarks, we first perform dimension reduction using MVU or PCA to get an LD representation $\tilde{Y}$, and then regress it on $X$. For all methods, the process optimization is carried out following the procedure described in Section \ref{sectionopt}. A visual comparison of the LD representation $\tilde{Y}$ along with the reconstructed $\tilde{Y}$ from the linear regression on $X$ is shown in Figure \ref{fig:swisscomp}. In this figure, the first row shows the unfolded manifolds to a 2D coordinates, i.e., $\tilde{Y}$, using MCU, MVU, and PCA. The second row shows the predicted $\tilde{Y}$ from the linear regression model. We use a colormap to facilitate understanding the unfolding process by comparing them with the color of the manifold shown in Figure \ref{fig:swissmani}.

Note that an appropriate manifold learning/unfolding method should unfold the manifold to a rectangular shape so its variation can linearly be explained by the explanatory variables. As can be seen from the plots,  only MCU can unfold the manifold into a rectangular shape whose sides are in parallel with the $x$ and $y$ axes. This is mainly because the MCU finds an LD representation that maximizes the covariance between the control variables and unfolded manifold. PCA has the worst unfolding performance as the unfolded manifold has a nonlinear relationship with the control variables. 

Figure \ref{fig:swisscompdist} also shows the pointwise deviation between $\tilde{Y}$ and reconstructed $\tilde{Y}$ plotted on a heatmap. From this figure, it is clear that the resulting regression model by the MCU and MVU clearly has lower residuals and higher predictability. To compare the prediction performance, we also plot the boxplots of the RRE obtained by each method in Figure \ref{fig:swissregcomp}. As can be seen from the figure, both MCU and MVU outperform PCA. Additionally, although the median for MCU and MVU are similar, the Inter Quartile Range (IQR) of MCU (0.1538) is smaller than that of MVU (0.1781). \\

%on the reconstructed $\tilde{Y}$ grid, and , where the mean $\tilde{Y}$ reconstruction error using linear regression on X is smallest comparing to other two methods, thanks to that the MCU takes the covariance between X and Y into consideration. \\

\begin{figure}[htbp]
    \centering
    \includegraphics[scale=0.45]{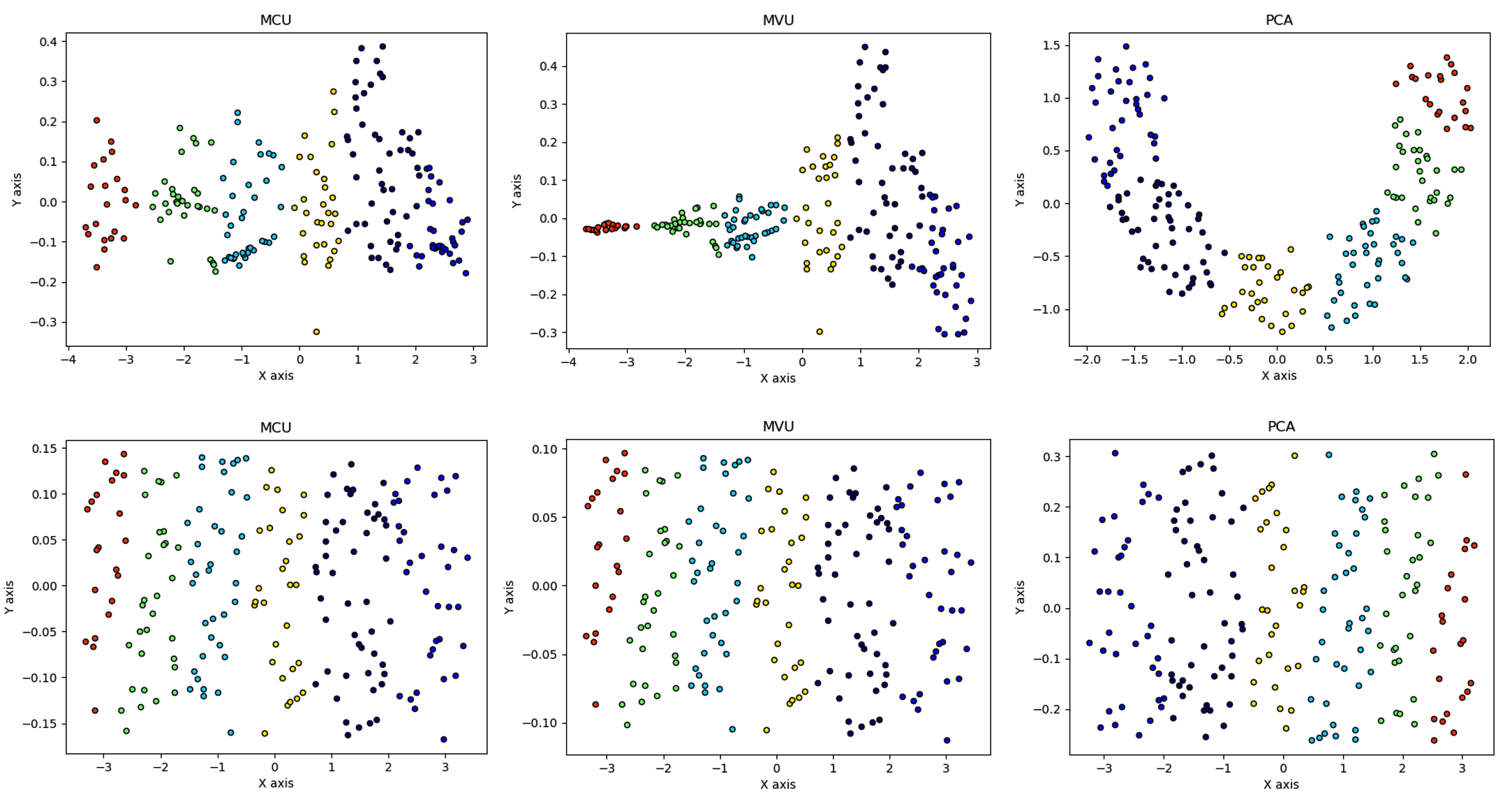}
    \caption{\label{fig:swisscomp}A visual comparison of three methods in terms of unfolding on Swiss roll data (scaled)}
\end{figure}{} 

\begin{figure}[htbp]
    \centering
    \includegraphics[scale=0.45]{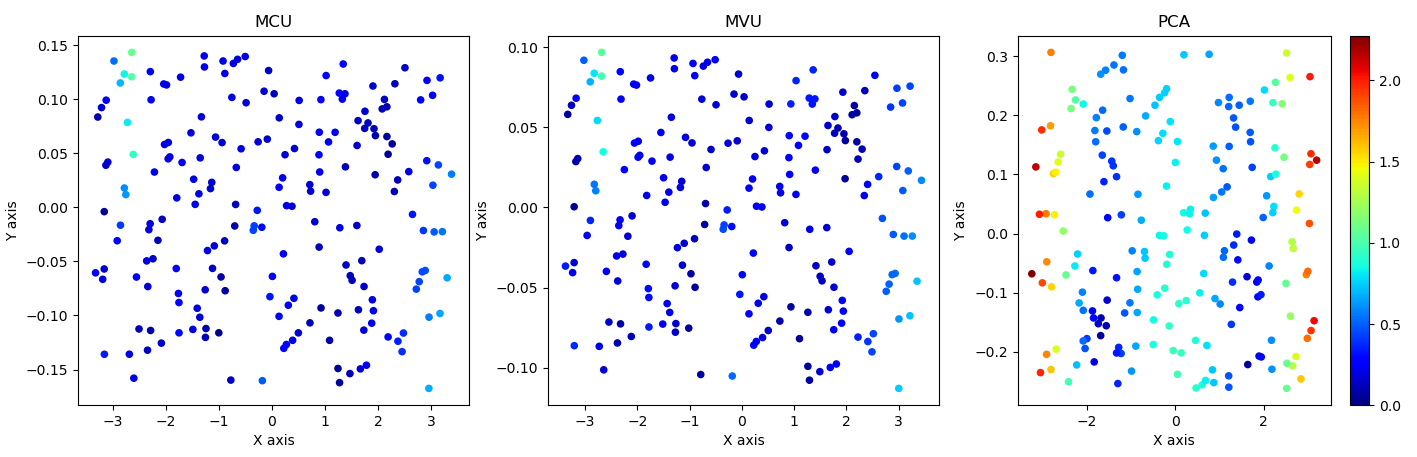}
    \caption{\label{fig:swisscompdist}A visual comparison of three methods in terms of pointwise $\tilde{Y}$ reconstruction deviation (scaled)}
\end{figure}{} 

\begin{figure}[htbp]
\centering
\includegraphics[scale=0.5]{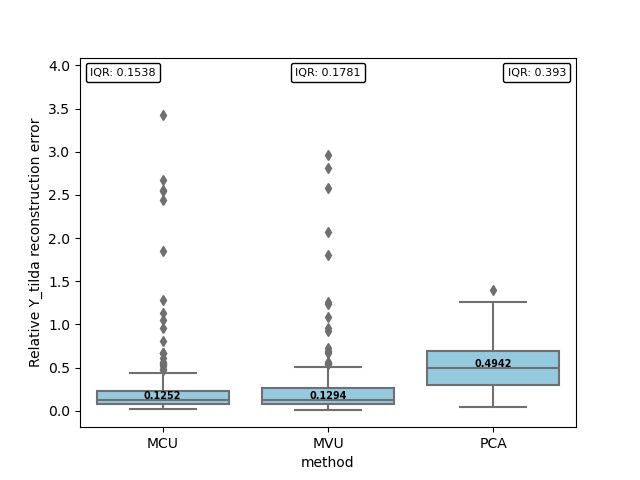}
    \caption{\label{fig:swissregcomp}A comparison of $\tilde{Y}$ reconstruction error on Swiss roll data}
\end{figure}

Finally, we test our algorithm on a unseen nominal response to see if we can successfully recover the optimal control variables via the proposed process optimization. To this end, we generate a nominal response vector, denoted by $y_{nom}$, and record the corresponding control variables used for generating the nominal response as the true optimal setting, denoted by $x_{nom}$. Based on this nominal sample, we run our process optimization algorithm to see how close the computed optimal setting $\hat{x}_{nom}$ are to the true optimal setting $x_{nom}$. We, first, evaluate the Euclidean norm between $\hat{x}_{nom}$ and $x_{nom}$ of the three methods as shown in Figure \ref{fig:swissXopt}. From the figure, it is clear that the optimal setting obtained by MCU is the closer to the true optimal setting than MVU and PCA by 18\%\ and 75\%\ . The pointwise deviations between the response $\hat{y}_{nom}$ under the computed optimal setting and the true nominal $y_{nom}$ are also shown in Figure \ref{fig:swissYopt} as a boxplot. As can be seen MCU with the median deviations of $10.04$ outperforms both benchmarks, MVU and PCA with the median deviations of $11.0$ and $28.5$, respectively.   

\begin{figure}[htbp]
\begin{minipage}[t]{0.5\linewidth}
\centering
\includegraphics[scale=0.5]{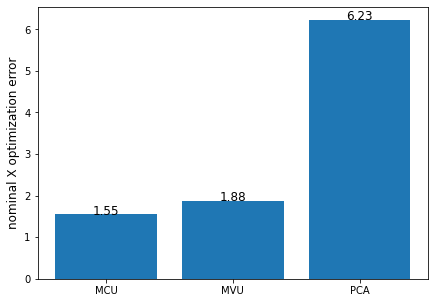}
    \caption{\label{fig:swissXopt}A comparison of deviation between $\hat{x}_{nom}$ and $x_{nom}$ on Swiss roll data}
\end{minipage}%
\begin{minipage}[t]{0.5\linewidth}
\centering
\includegraphics[scale=0.5]{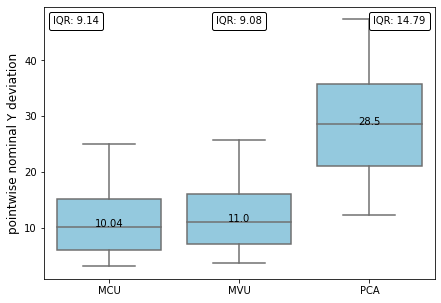}
    \caption{\label{fig:swissYopt}A comparison of pointwise deviation between $\hat{y}_{nom}$ and $y_{nom}$ on Swiss roll data}
\end{minipage}%
\end{figure}

\subsection{Penny Image}
To generate the Penny Image dataset, we begin with a base penny image as shown in Figure \ref{fig:penny1}, and apply two operations, namely, rotation and translation, to the image to generate variations. The two operations are defined via two control (explanatory) variables $c_1$ and $c_2$, where $c_1$ is used to control the rotational angle and $c_2$ is to control the translated position of the generated penny image samples. These operations are defined by
$$
rotational\ degree = -4(c_1-5) \ counterclockwise,
$$
$$
translation\ coordinates = \left(\frac{(c_2-5)l}{40}, \frac{l}{8}sin\left(\frac{\pi(c_2-5)}{10}\right)\right)/10,
$$
where $l$ is the length of the square penny image, which is 55 in our case. We, then, generate 100 independent penny images by varying $c_1$ and $c_2$ uniformly from the $[1, 10]$ interval. Each sample is a $55\times 55$ image, so after concatenating the intensity of these 3025 pixels in a row, the resulting response matrix $Y$ is of size $N\times 3025$, where $N = 100$ is the sample size. The explanatory variable matrix $X$ is of size $N\times 2$. A demonstration of generated pennies with varying $c_1,c_2$ is shown in Figure \ref{fig:penny2}.\\ 

\begin{figure}[htbp]
\centering
\includegraphics[scale=0.5]{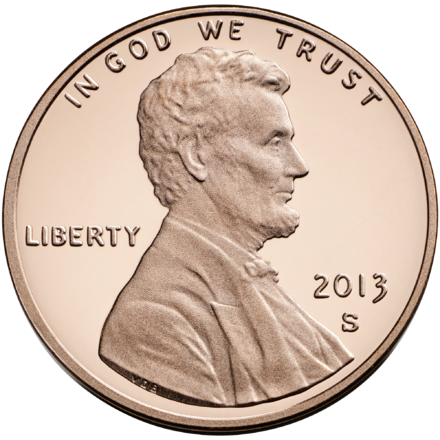}
    \caption{\label{fig:penny1}Original Penny Image}
\end{figure}

\begin{figure}[htbp]
\centering
\includegraphics[scale=0.25]{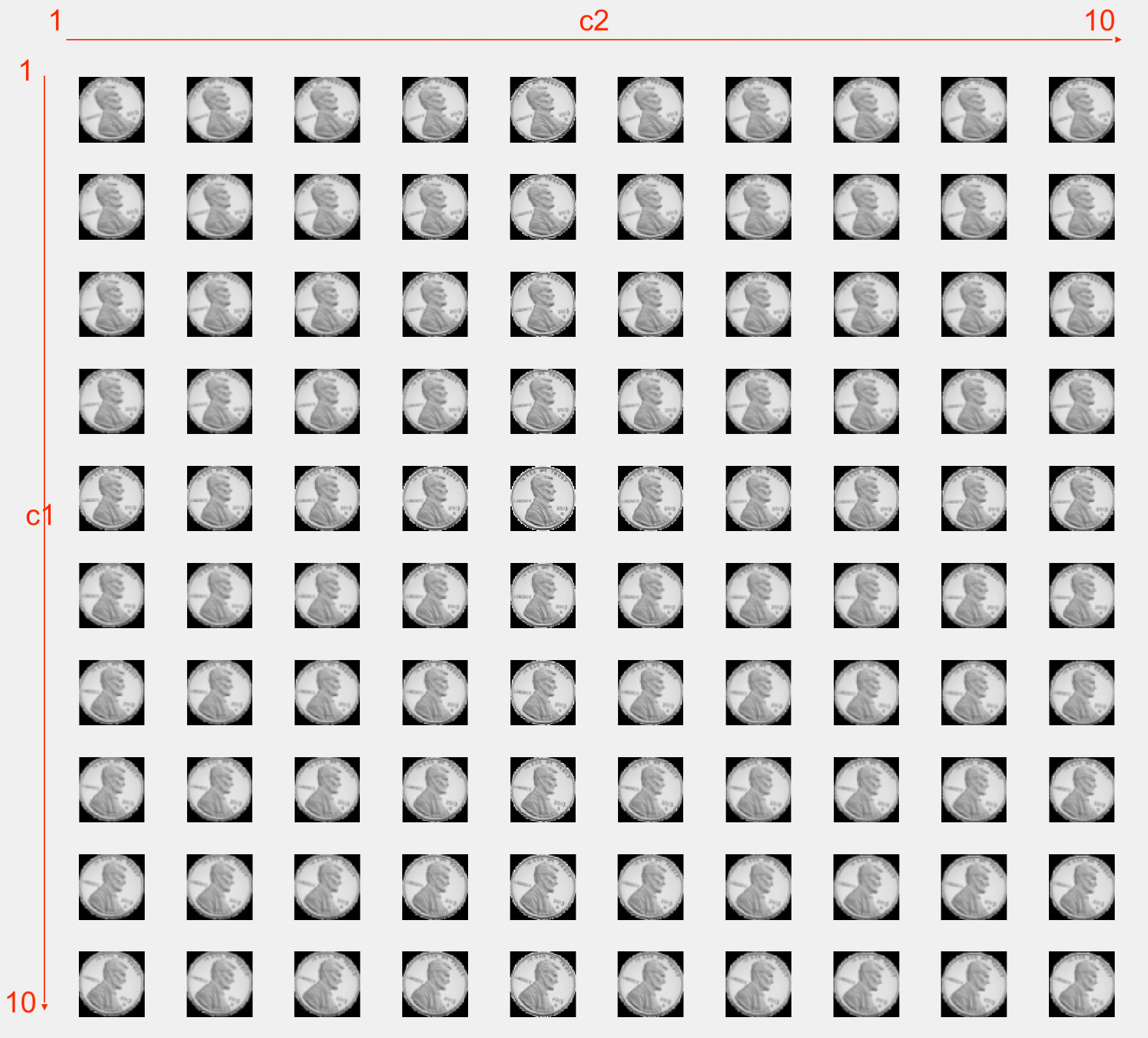}
    \caption{\label{fig:penny2}Penny Images with varying c1 and c2}
\end{figure}

To study the shape of the generated manifold, we first apply a PCA to the set of generated penny samples and visualize the first 3 PC scores in Figure \ref{fig:pennypc}. The arrows on the plot show how the control variables $(c_1,c_2)$ change the coordinates of the samples on the generated manifold. As can be seen from the plot, the manifold has a nonlinear shape like a hump, where $c_1$ controls the height of the hump and $c_2$ controls the width.\\

\begin{figure}[htbp]
\centering
\includegraphics[scale=0.4]{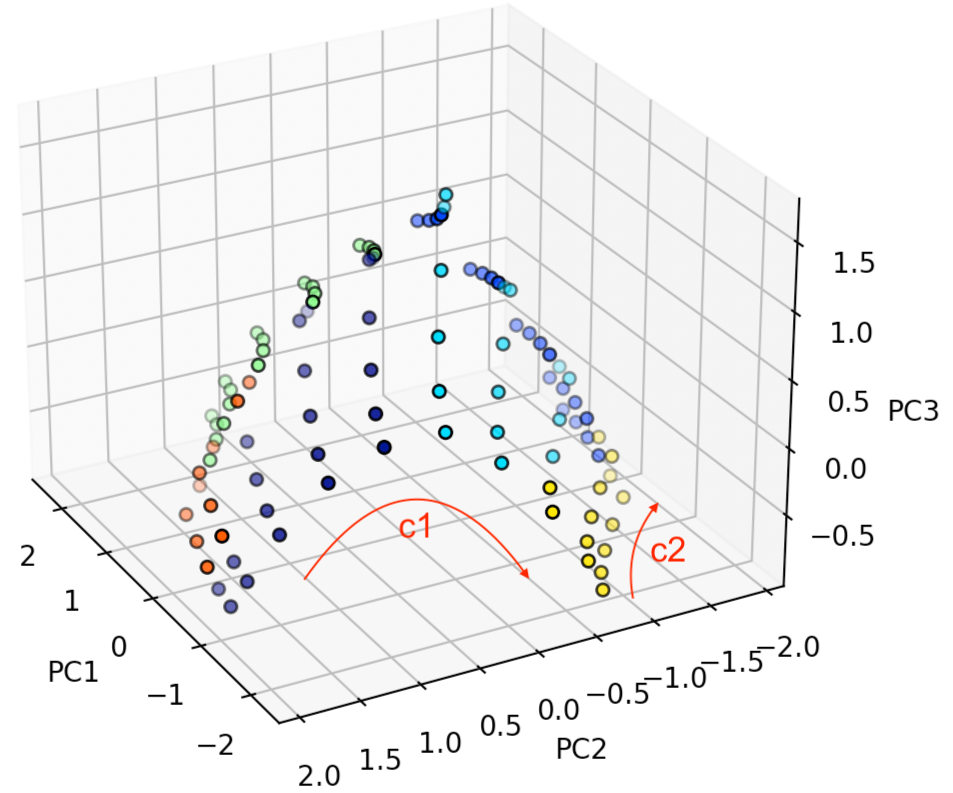}
    \caption{\label{fig:pennypc}A colormap(based on clustering) of manifold shape of Penny Image data}
\end{figure}

% commenting out the following block
%\begin{comment}
%We run our algorithm using the hyperparameter setting as follows. 
%$$
% k = 4, \lambda = 10^{-10}, d = 2
%$$

%The degrees of freedom of the underlying data manifold should be 2 as a result of the number of control variables being 2, thus we choose $d$ as 2 in the hyperparameter selection.\\
%\end{comment}

Next, we apply the proposed MCU along with the benchmarks, i.e. PCA and MVU, and compare their performance. For the benchmarks, we first perform dimension reduction using MVU or PCA to get LD features $\tilde{Y}$, and then regress them on $X$. For all the methods, the process optimization is carried out following the procedure described in Section \ref{sectionopt}. A visual comparison of LD representation $\tilde{Y}$ with reconstructed $\tilde{Y}$ from the linear regression on X is shown in Figure \ref{fig:pennycomp}. As can be seen from the plots, similar to the previous case, only MCU can unfold the manifold into a rectangular shape whose sides are in parallel with the $x$ and $y$ axes. PCA has the worst unfolding performance as the unfolded manifold has a nonlinear relationship with the control variables.

The pointwise residual plot in Figure \ref{fig:pennycompdist} indicates that the MCU clearly has lower residuals and higher predictability. To compare the prediction performance, we also plot the boxplots of RRE in Figure \ref{fig:pennyregcomp}. As can be seen from the figure, the proposed MCU outperforms MVU and PCA in both the median and IQR. For example, the median RRE of MCU is 0.035 compared with those of MVU and PCA, 0.081 and 0.119, respectively.   

\begin{figure}[htbp]
    \centering
    \includegraphics[scale=0.35]{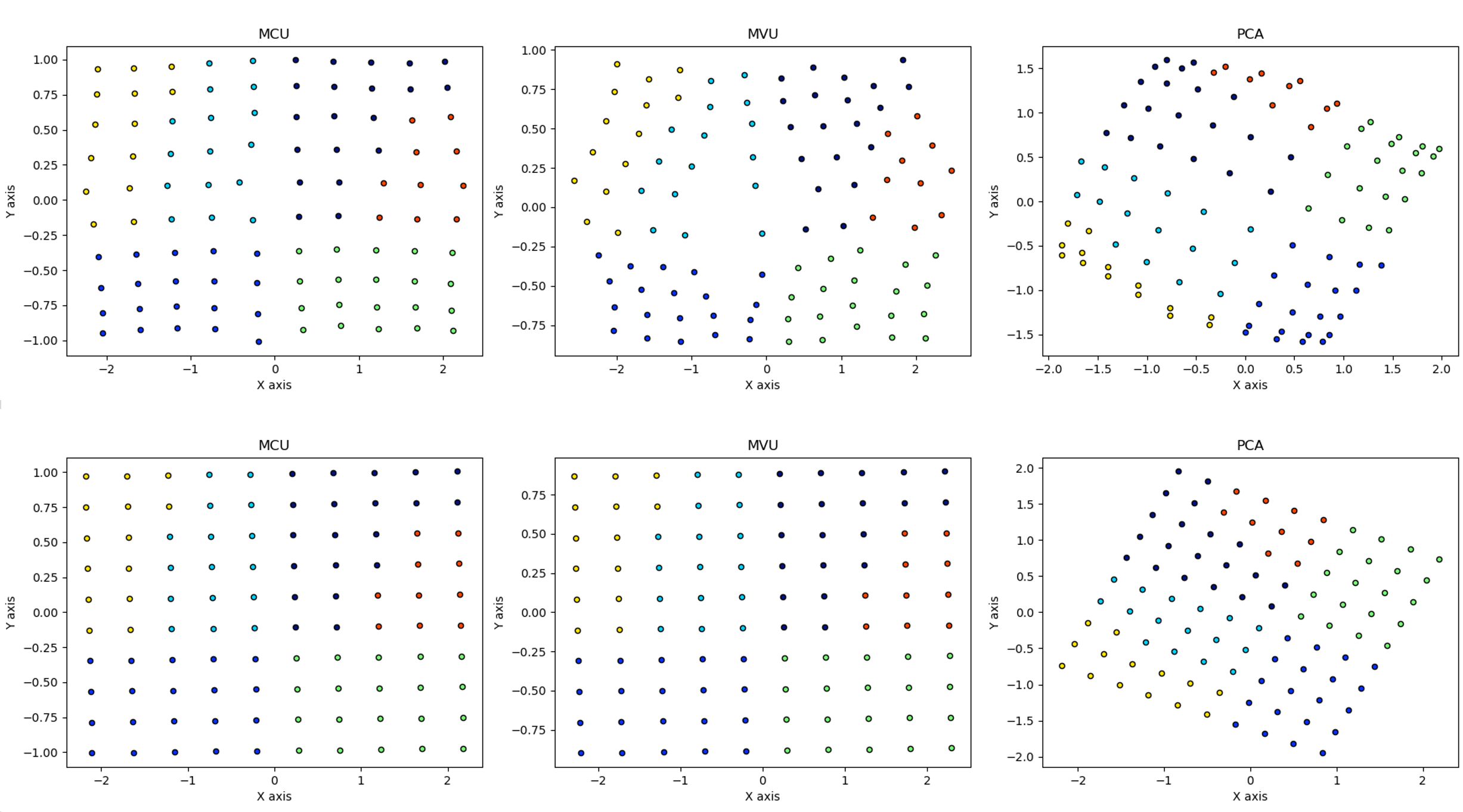}
    \caption{\label{fig:pennycomp}A visual comparison of three methods in terms of unfolding on penny image data (scaled)}
\end{figure}{} 

\begin{figure}[htbp]
    \centering
    \includegraphics[scale=0.45]{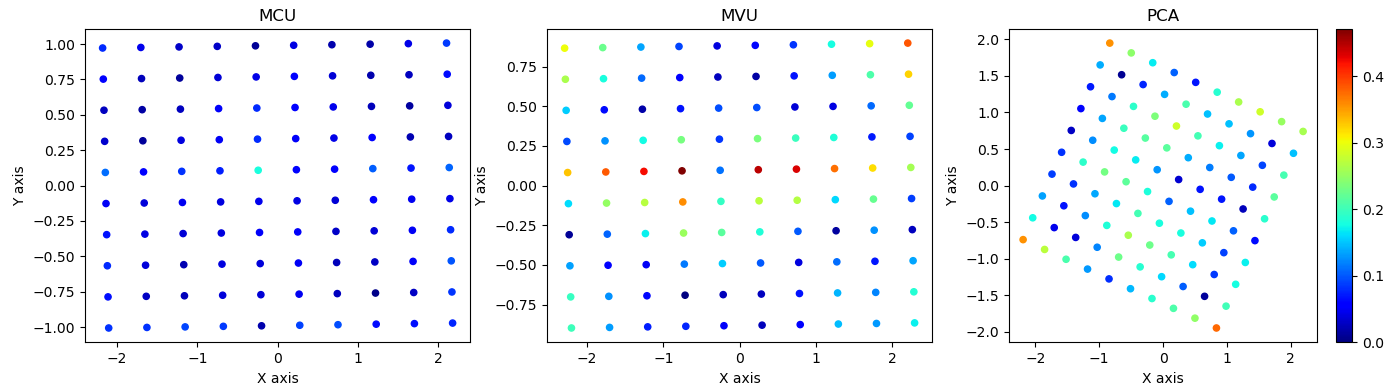}
    \caption{\label{fig:pennycompdist}A visual comparison of three methods in terms of pointwise $\tilde{Y}$ reconstruction deviation (scaled)}
\end{figure}{} 

\begin{figure}[htbp]
\centering
\includegraphics[scale=0.5]{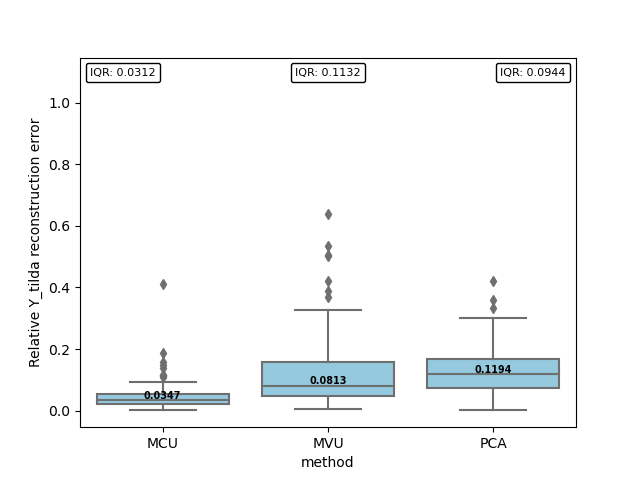}
    \caption{\label{fig:pennyregcomp}A comparison of $\tilde{Y}$ reconstruction error on Penny Image data}
\end{figure}

Next, we validate the proposed process optimization method following the procedure described in the Swiss roll case. Using the nominal $y_{nom}$, we compute the optimal setting for the control variables, $\hat{x}_{nom}$, and compare them with the true optimal, $x_{nom}$. We first compare the Euclidean norm between $\hat{x}_{nom}$ and $x_{nom}$ of all three methods shown in Figure \ref{fig:pennyXopt}. The MCU results in the closest optimal setting to the true one with the distance of 0.08 compared with 0.45 and 0.28 of MCU and PCA. We also study the pointwise deviation between the response $\hat{y}_{nom}$ and the true nominal $y_{nom}$ by plotting their boxplots in Figure \ref{fig:pennyYopt}. The boxplots show that MCU outperforms both MCU and PCA in terms of the median deviation as well as IQR. Note that since the pixels outside of the perimeter of the penny images (white pixels) are zero before and after the operations, those pixels are excluded when the boxplots are plotted to better visualize the performance of the methods. 

\begin{figure}[htbp]
\begin{minipage}[t]{0.5\linewidth}
\centering
\includegraphics[scale=0.5]{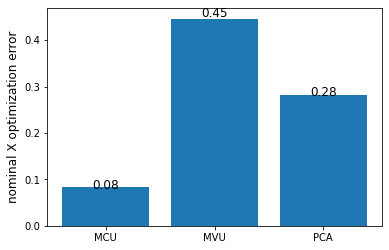}
    \caption{\label{fig:pennyXopt}A comparison of deviation between $\hat{x}_{nom}$ and $x_{nom}$ on Penny Image data}
\end{minipage}%
\begin{minipage}[t]{0.5\linewidth}
\centering
\includegraphics[scale=0.5]{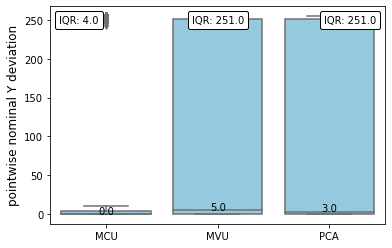}
    \caption{\label{fig:pennyYopt}A comparison of pointwise deviation between $\hat{y}_{nom}$ and $y_{nom}$ on Penny Image data}
\end{minipage}%
\end{figure}

\section{Case study}
In this section, we validate the proposed method using a case study on metal brackets manufacturing \cite{dataset}. Figure \ref{fig:bracketnom} shows the nominal dimensions of the metal bracket and Figure \ref{fig:bracketpc} shows the plot of one sample point cloud obtained from a manufactured bracket.  In this process, the point clouds obtained from the metal brackets were scanned via a Faro Laser ScanArm v2.0, \citep{dataset}. 

\begin{figure}[htbp]
\begin{minipage}[t]{0.5\linewidth}
\centering
\includegraphics[scale=0.35]{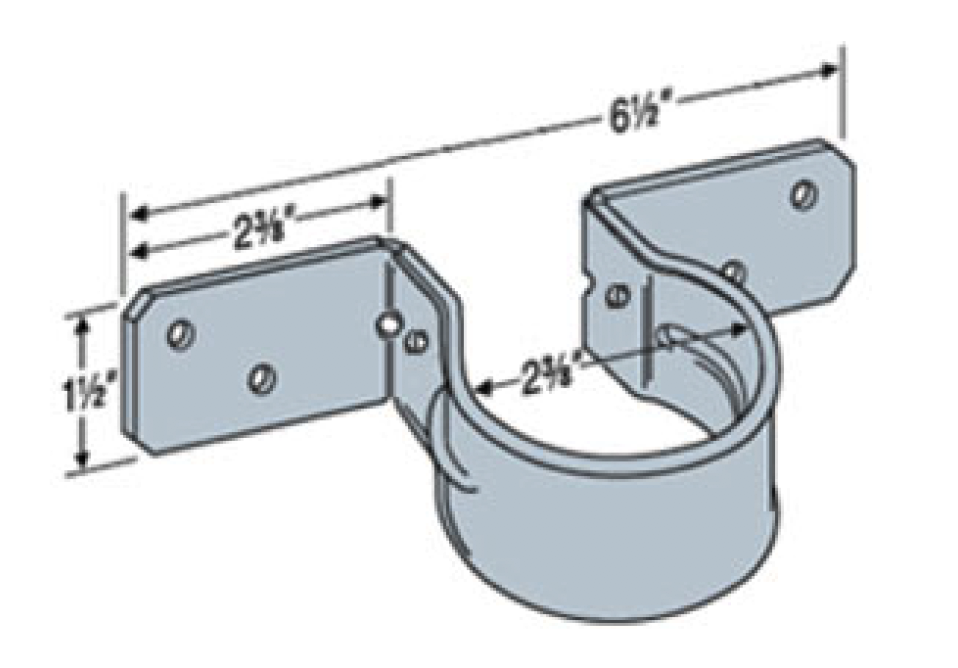}
    \caption{\label{fig:bracketnom}Nominal bracket}
\end{minipage}%
\begin{minipage}[t]{0.5\linewidth}
\centering
\includegraphics[scale=0.35]{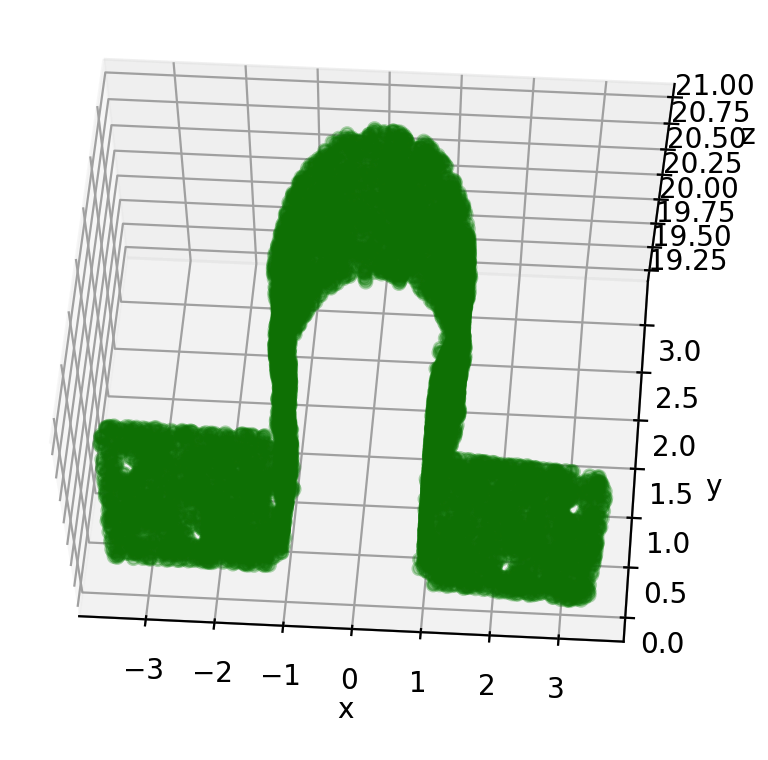}
    \caption{\label{fig:bracketpc}The point cloud data of one sample bracket}
\end{minipage}%
\end{figure}

We imitate the deformation mode described in  \cite{dataset} to generate a set of point clouds under different conditions using the nominal bracket in Figure \ref{fig:bracketpc}. For this purpose, we define a function $f: \mathbf{R}^3\rightarrow \mathbf{R}^3$ to create a deformed shape of the bracket via two explanatory variables $c_1$ and $c_2$. Specifically, $c_1$ is used to control the bending degree of the bracket’s mounting plate, and $c_2$ is used to control the $x$-coordinate of the bending pivot on the mounting plate. The deformation function used to generate the samples is given by
$$
f(x,y,z) = (x,\ y+10[1-cos((1+c_1/15)(x-pivot)/4)],z) \text{\ \ if\ }x>pivot,
$$
$$
f(x,y,z) = (x, y, z) \text{\ \ otherwise\ },
$$
where $pivot = 1+0.15c_2$.\\

A graphical illustration of the control variables' effect on the bending deformation can be seen in Figure \ref{fig:brktbend}, where the point cloud data along the $z$-axis are projected onto the $x$-$y$ plane for better visualization. By randomly selecting $c_1$ and $c_2$ from $[1, 10]$, we generate 200 bracket samples, each of which contains $10,000$ datapoints. The concatenation of the coordinates of these $10,000$ datapoints in a row results in a matrix $Y$ of size $200\times 30,000$, where $200$ is the sample size. The explanatory variable matrix $X$ is of size $200\times 2$.\\

\begin{figure}[htbp]
    \centering
    \includegraphics[scale=0.4]{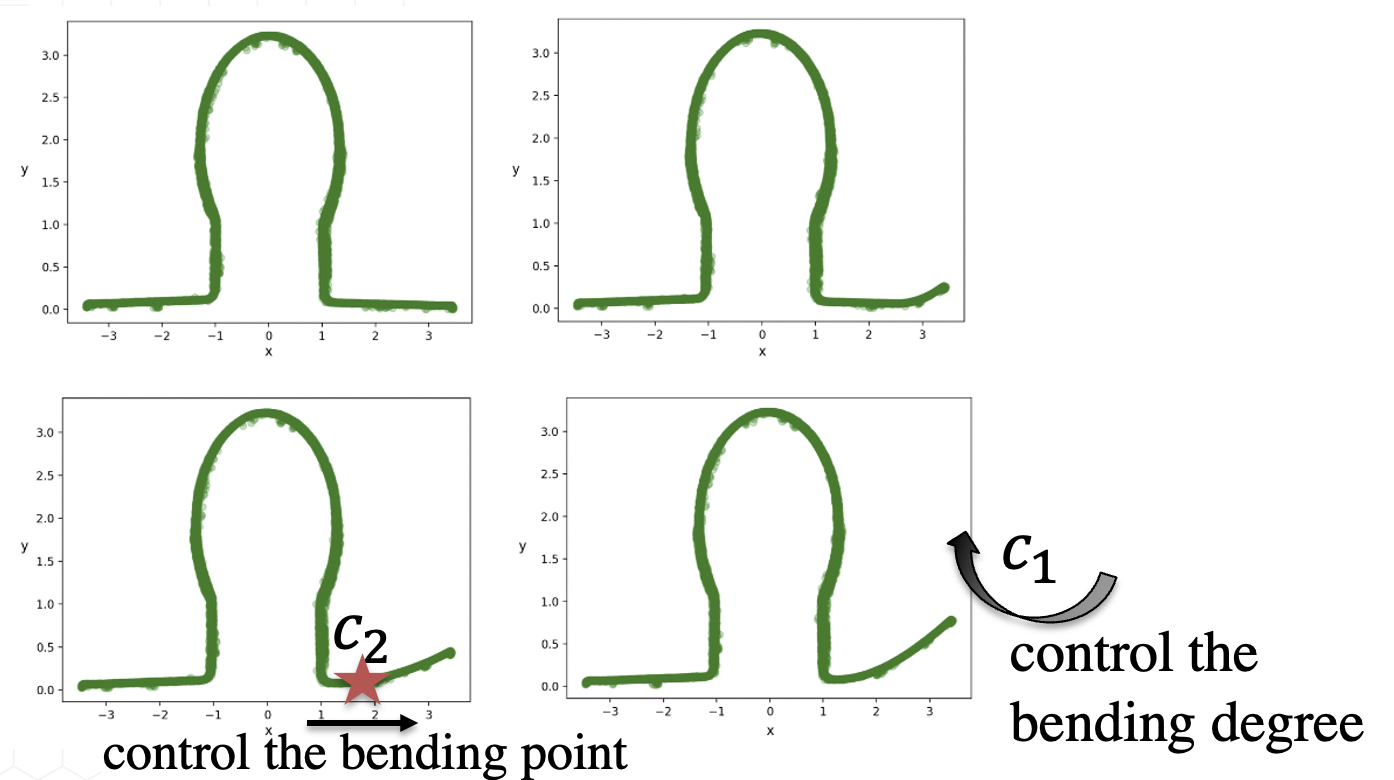}
    \caption{\label{fig:brktbend}A visualization of effects of control variables on metal bracket data}
\end{figure}{}

We first apply PCA to the group of generated bracket samples to investigate the variation pattern of the underlying manifold structure. The visualization of the first 3 PC scores is shown in Figure \ref{fig:brktmani}, where the arrows show the way the explanatory variables $(c_1,c_2)$ change the coordinates of the samples on the generated manifold. As can be seen from the plot, the manifold has the shape of a cloak, where $c_1$ controls the length of the cloak and $c_2$ controls the width.

\begin{figure}[htbp]
\centering
\includegraphics[scale=0.4]{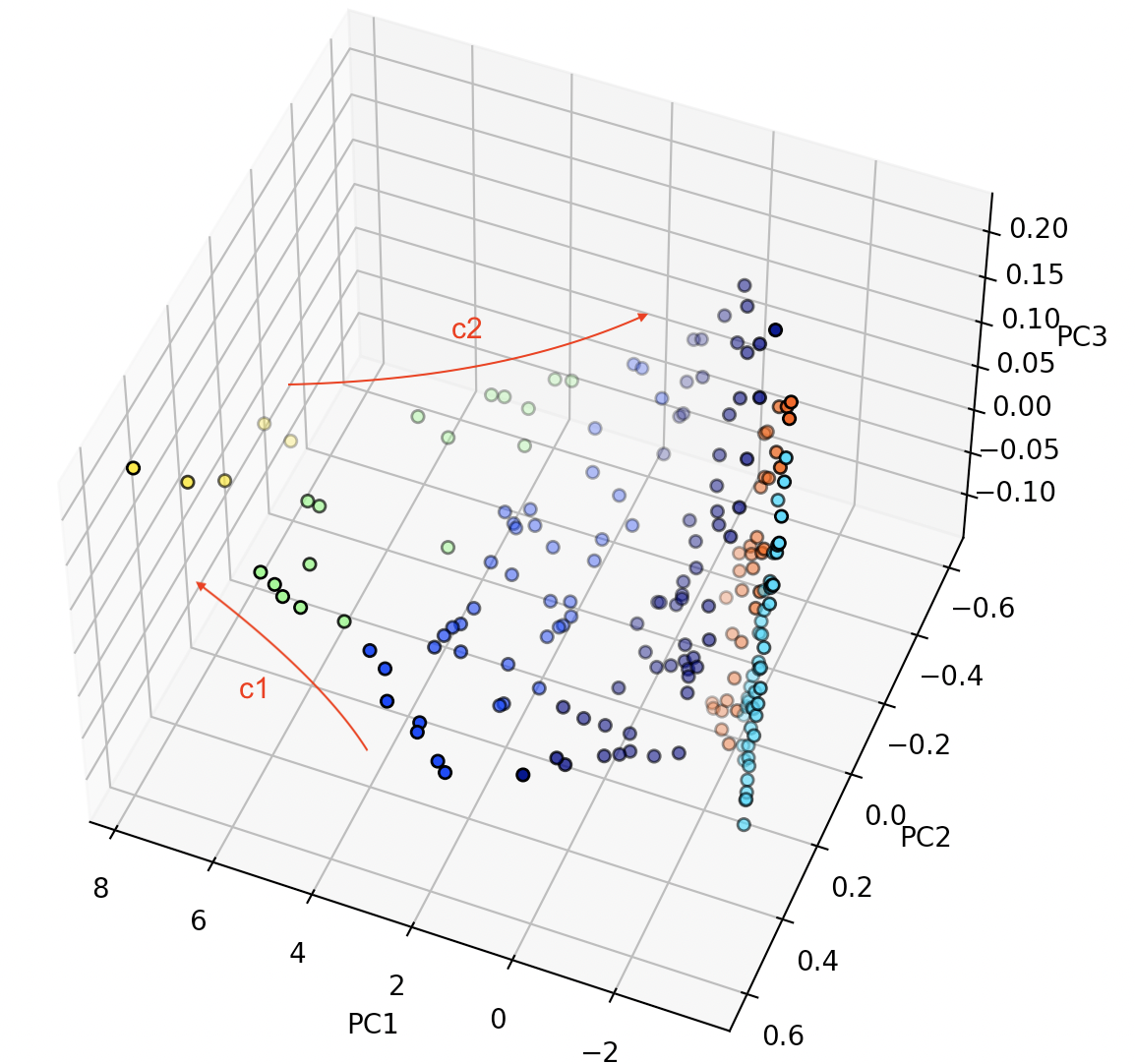}
    \caption{\label{fig:brktmani}A colormap(based on clustering) of manifold shape of bracket data}
\end{figure}

%We run our algorithm using the hyperparameter setting as follows. 
%$$
% k = 4, \lambda = \text{1e-10}, d = 2
%$$

Next, we compare the performance of our proposed MCU with the benchmarks PCA and MVU on the generated data. For the benchmarks, we first reduce the data dimension using MVU/PCA before regressing the obtained LD representation $\tilde{Y}$ on $X$. In Figure \ref{fig:bracketcomp}, the first row shows the 2D unfolded manifolds $\tilde{Y}$ using MCU, MVU, and PCA, which gives us a high level understanding of the unfolding behavior. The second row shows the predicted $\tilde{Y}$ from the linear regression model on $X$. As dicussed earlier, a perfect manifold learning method should unfold the manifold to a rectangular shape so that its variation can be explained by the explanatory variables in a linear way. As can be seen from the plots, MCU unfolds the manifold into the most rectangular shape. Though this rectangle is not perfect and subject to a linear transformation, it can be captured by the subsequent linear regression on $X$. MCU has the best performance as a result of its capability of finding a LD representation that maximizes the covariance between the control variables and unfolded manifold. PCA has the worst unfolding performance as the unfolded manifold has a very nonlinear shape. 

\begin{figure}[htbp]
    \centering
    \includegraphics[scale=0.4]{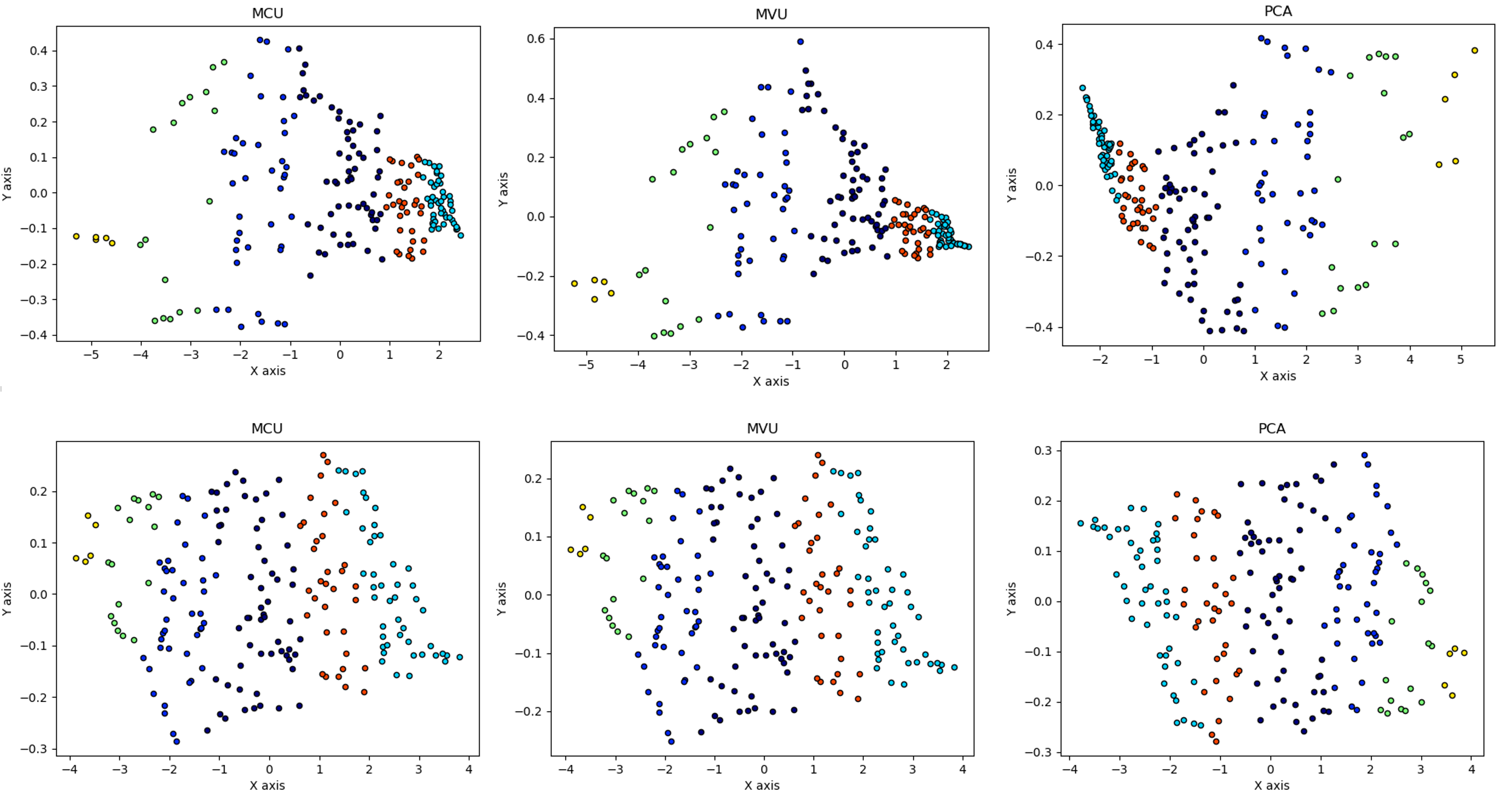}
    \caption{\label{fig:bracketcomp}A visual comparison of three methods in terms of unfolding on bracket data (scaled)}
\end{figure}{} 

\begin{figure}[htbp]
    \centering
    \includegraphics[scale=0.35]{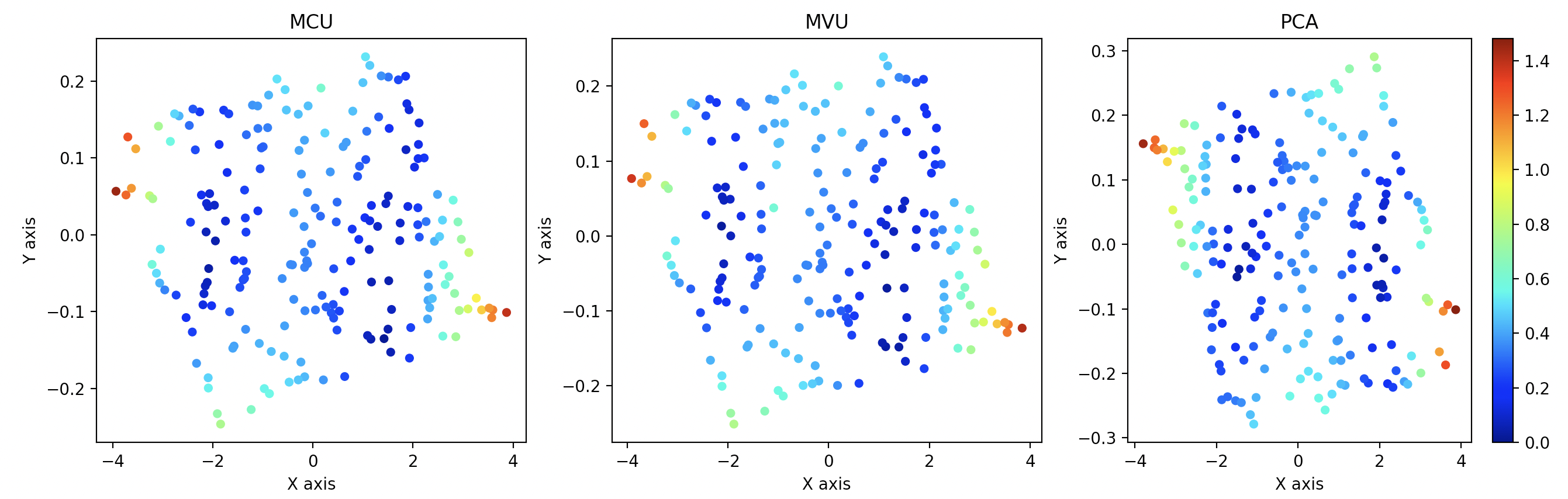}
    \caption{\label{fig:brktcompdist}A visual comparison of three methods in terms of pointwise $\tilde{Y}$ reconstruction deviation (scaled)}
\end{figure}{} 

The pointwise deviation between $\tilde{Y}$ and reconstructed $\tilde{Y}$ is plotted in Figure \ref{fig:brktcompdist} as a heatmap. Clearly the regression models from MCU and MVU have lower residuals and higher prediction accuracy. To quantify the prediction performance, the boxplots of the RRE = $\|\tilde{y}-\hat{\tilde{y}}\|_2/\|\tilde{y}\|_2$ of each method are plotted in Figure \ref{fig:brktregcomp}. From the figure, we can readily see that both MCU and MVU outperform PCA. In addition, the IQR of MCU (0.3024) is smaller than that of MVU (0.3344). 

\begin{figure}[htbp]
\centering
\includegraphics[scale=0.5]{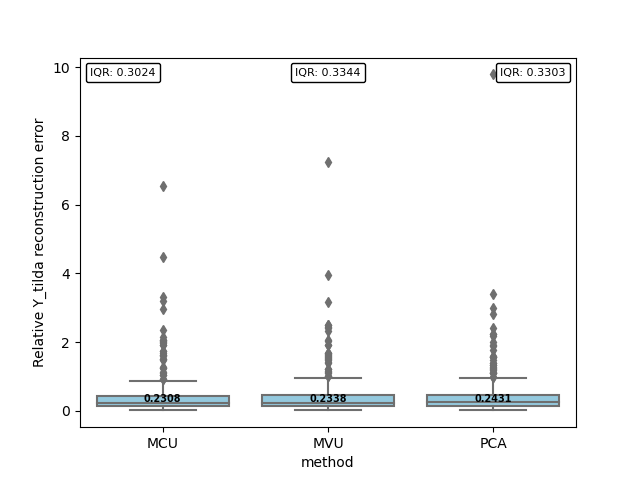}
    \caption{\label{fig:brktregcomp}A comparison of $\tilde{Y}$ reconstruction error on bracket data}
\end{figure}

Subsequently, we test our algorithm on the nominal response $y_{nom}$ to see if the optimal process setting $x_{nom}$ can be successfully recovered via a process optimization process as described in Section \ref{sectionopt}. We first evaluate the Euclidean distance between the learned control setting $\hat{x}_{nom}$ and the nominal $x_{nom}$ as in Figure \ref{fig:brktXopt}. As a second evaluation criterion, Figure \ref{fig:brktYopt} shows the boxplots of the pointwise deviation between the response $\hat{y}_{nom}$ under the optimized setting and the true nominal $y_{nom}$. Note that the locations of most of the points in the bracket are unaltered after the deformation except those on the right mounting plate. This leads to a large number of zero pointwise deviations in the boxplot. We, hence, filter them out and only include the deviations in the right mounting plate area for a better visualization. As can be seen from Figure \ref{fig:brktXopt} and \ref{fig:brktYopt}, MCU outperforms its counterpart MVU and the benchmark PCA in terms of both evaluation criteria.  

\begin{figure}[htbp]
\begin{minipage}[t]{0.5\linewidth}
\centering
\includegraphics[scale=0.5]{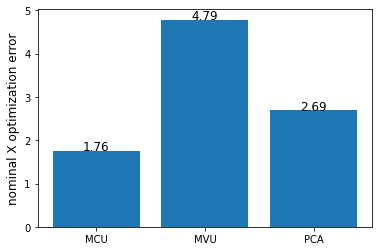}
    \caption{\label{fig:brktXopt}A comparison of deviation between $\hat{x}_{nom}$ and $x_{nom}$ on bracket data}
\end{minipage}%
\begin{minipage}[t]{0.5\linewidth}
\centering
\includegraphics[scale=0.5]{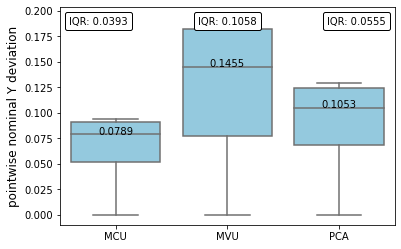}
    \caption{\label{fig:brktYopt}A comparison of pointwise deviation between $\hat{y}_{nom}$ and $y_{nom}$ on bracket data}
\end{minipage}%
\end{figure}

\section{Conclusion}
In this paper, we proposed a novel manifold learning approach called MCU Regression that is able to effectively learn LD structure of HD point clouds, and use it to quantify the association between the point clouds and some explanatory/control variables. This relationship can further be used for the process optimization. Inspired by the MVU \citep{666666}, the proposed MCU unfolds the data manifold by maximizing the covariance between regressors (covariates) and the point cloud response to make sure the useful association information is preserved, while the LD structure is learned. 

We performed simulation studies and applied the MCU regression on a bracket manufacturing dataset to validate the proposed framework and to compare its performance with two commonly used tools, namely, PCA and MVU. The results of both simulation and the case study indicated that the proposed MCU can unfold the manifold more effectively than the benchmarks, and preserving its correlation with explanatory variables. Additionally, from the study we observed that the resulting optimal setting from the proposed method is closer to the true optimal setting than those generated by the benchmarks. Extension of the proposed method to SPC applications for monitoring point-cloud profiles is practically important, yet challenging topic that requires further research.

\bigskip

\bibliographystyle{chicago}
\bibliography{Bibliography-MM-MC}

%%%%%%%%%%%%%%%%%
\end{document}